\documentclass[twoside,11pt]{article}

\usepackage{jmlr2e}

\ShortHeadings{Theoretical Guarantees for Domain Adaptation with Hierarchical Optimal Transport}{El Hamri, Bennani, Falih}
\firstpageno{1}

\begin{document}

\title{Theoretical Guarantees for Domain Adaptation with Hierarchical Optimal Transport}

\author{\name Mourad El Hamri \email mourad.elhamri@sorbonne-paris-nord.fr \\ 
       \addr LIPN, UMR CNRS 7030\\
       La Maison des Sciences Numériques\\
       Université Sorbonne Paris Nord\\
       Paris, France\\
        \AND
       \name Younès Bennani  \email younes.bennani@sorbonne-paris-nord.fr \\
       \addr LIPN, UMR CNRS 7030\\
       La Maison des Sciences Numériques\\
       Université Sorbonne Paris Nord\\
       Paris, France\\
       \AND
       \name Issam Falih  \email issam.falih@uca.fr \\
       \addr  LIMOS, UMR CNRS 6158\\
       Université Clerment Auvergne\\
       Clermont-Ferrand, France\\}

\editor{\textbf{Under review}}

\maketitle

\begin{abstract}
Domain adaptation arises as an important problem in statistical learning theory when the data-generating processes differ between training and test samples, respectively called 
source and target domains. Recent theoretical advances show that the success of domain adaptation algorithms heavily relies on their ability to minimize the divergence between the probability distributions of the source and target domains. However, minimizing this divergence cannot be done independently of the minimization of other key ingredients such as the source risk or the combined error of the ideal joint hypothesis. The trade-off between these terms is often ensured by algorithmic solutions that remain implicit and not directly reflected by the theoretical guarantees. To get to the bottom of this issue, we propose in this paper a new theoretical framework for domain adaptation through hierarchical optimal transport.  This framework provides more explicit generalization bounds and allows us to consider the natural hierarchical organization of samples in both domains into classes or clusters. Additionally, we provide a new divergence measure between the source and target domains called Hierarchical Wasserstein distance that indicates under mild assumptions, which structures have to be aligned to lead to a successful adaptation.
\end{abstract}

\section{Introduction}
Supervised learning is, beyond a reasonable doubt, the most studied theoretical framework of machine learning. Many of these theoretical studies are concerned with estimating the probability that a  specific hypothesis can achieve a small true risk. Uniform convergence theory guarantees the expression of such a probability in the guise of generalization bounds on the true risk, under the overriding presumption that training and test samples are drawn from the same probability distribution. However, the theoretical results of supervised learning do not cover some real-world scenarios where the data-generating processes differ between the training and test samples. Such a predicament has fostered the emersion of domain adaptation \citep{redko2019advances}, a branch of statistical learning theory \citep{vapnik1999nature} that considers the distributional shift between training and test samples, respectively named the source and target domains. In domain adaptation theory, existing generalization bounds on the target risk of a given hypothesis are often stated in a generic form implying the source risk, a divergence measure between the source and target domains, and a term assessing the ability of the given hypothesis space to successfully resolve the problem of adaptation. The source risk is estimable from finite samples and can be minimized by learning the hypothesis from the available source labeled data. Similarly, the divergence is estimable from the observed data and is intended to be slight if the two domains are nearby. While the last term is non-estimable and is usually formulated as the combined error of the ideal joint hypothesis.
\\
\\In the pioneering theoretical work of \citep{ben2006analysis}, the domain adaptation problem was carefully addressed using the total variation TV distance, but its employment as a divergence measure between the marginal distributions of the source and target domains presents two major weaknesses. First, the TV distance is not directly related to the concerned hypothesis space, which results in loose generalization bounds, and secondly, it is not estimable from finite samples drawn from arbitrary probability distributions \citep{batu2000testing}. To overcome these limitations, \citep{ben2010theory} introduced a classifier-induced divergence called the $\mathcal{H}\Delta\mathcal{H}$-divergence, based on the $\mathcal{A}$-divergence provided in (Kifer et al., 2004). Indeed, the $\mathcal{H}\Delta\mathcal{H}$-divergence explicitly considers the given hypothesis space  $\mathcal{H}$, which guarantees that the generalization bounds stay relevant and decidedly linked to the learning problem in question, and, for a given hypothesis space $\mathcal{H}$ of finite Vapnik-Chervonenkis dimension, the $\mathcal{H}\Delta\mathcal{H}$-divergence can be estimated from finite samples. Furthermore, the $\mathcal{H}\Delta\mathcal{H}$-divergence is always smaller than the TV distance for any hypothesis space $\mathcal{H}$, which results in tighter bounds. Nevertheless, an obvious shortcoming of the $\mathcal{H}\Delta\mathcal{H}$-divergence is its reliance on the 0 - 1 loss function. Whereas, it might be desirable to have generalization bounds for a more generic domain adaptation framework, where any arbitrary loss function with some suitable properties can be considered. To address this concern, \citep{mansour2009domain} introduced the discrepancy distance $disc_l$ that expands the previous theoretical analysis of domain adaptation for any arbitrary loss function, which is symmetric, bounded, and obeys the triangle inequality. Additionally, the discrepancy distance $disc_l$ relies on the hypothesis space $\mathcal{H}$, but the complexity term is rather related to the Rademacher complexity of $\mathcal{H}$. This distinctive refinement provides data-dependent bounds that are commonly sharper than those derived from Vapnik–Chervonenkis theory.
\\Despite their numerous advantages, both the $\mathcal{H}\Delta\mathcal{H}$-divergence and the discrepancy distance $disc_l$ suffer from a computational burden related to their estimation. In such a circumstance, it was natural to look for other metrics with some appealing computational properties to quantify the divergence between the two domains. Following this trend, \citep{redko2015nonnegative} appealed to the Maximum Mean Discrepancy (MMD) distance to infer generalization bounds analogous to that of \citep{ben2010theory}. These bounds turned out to be remarkably meaningful since an unbiased estimator of the squared MMD distance can be computed in linear time, and the complexity term does not depend on the Vapnik–Chervonenkis dimension but on the empirical Rademacher complexities of the hypothesis space with respect to the source and target samples. Some time later, \citep{redko2017theoretical} presented generalization bounds in terms of the Wasserstein distance $\mathcal{W}_1$ as a theoretical analysis of the seminal domain adaptation algorithm based on optimal transport \citep{courty2016optimal}. This analysis proved to be very fruitful for several reasons. First, the Wasserstein distance is computationally attractive, particularly in virtue of the entropic regularization introduced in \citep{cuturi2013sinkhorn}. Furthermore, the Wasserstein distance has the ability to capture the underlying geometry of the data in both domains. Moreover, the Wasserstein distance is quite strong, and according to \citep{villani2009optimal}, it is not so hard to associate the convergence information in the Wasserstein distance with certain smoothness bound to obtain convergence in stronger distances. This powerful asset of the Wasserstein distance gives tighter bounds compared to other results in state-of-the-art.
\\
\\ Under the above generic form of generalization bounds, it is clear that minimizing the previous distances between the marginal distributions of the source and target domains cannot be performed separately from minimizing the source risk and the ability term. For instance, the minimization of the Wasserstein distance results from the transport of the source to the target samples such that $\mathcal{W}_1$ becomes quite low when computing between the newly transported source samples and the target instances. Nevertheless, by minimizing the Wasserstein distance only, the obtained transformation may transport some source samples of different labels to the same target samples, and thus, the empirical source error cannot be adequately minimized. Moreover, the joint error will be negatively affected since no classifier will be capable of separating these source instances. We may also consider an ironically extreme situation of binary classification task where the transport plan sends the source data of each class to the target data of the inverse class. In such a case, the joint error will be drastically impacted. To avoid these pathological scenarios, a possible remedy was then to promote group sparsity in the optimal transport plan in order to restrict the source instances of different classes to be transported to the same target points. This algorithmic solution is implemented through a group-norm regularizer in \citep{courty2016optimal}. From a theoretical point of view, this regularization constitutes an arrangement to control the trade-off between the three terms of the bound. However, this trade-off remains imperceptible, and the bound does not reflect it explicitly. 
\\
\\Recently, \citep{el2022hierarchical} proposed a new domain adaptation algorithm based on a hierarchical formulation of optimal transport that leverages beyond the geometrical information captured by the ground metric, richer structural information in the source and target domains. The exploitation of this structural information elicited some desired properties in domain adaptation like preserving compact classes during the transportation, which provided an alternative algorithmic solution to restrict the source instances of different classes to be transported to the same target points. The main underlying idea behind the hierarchical formulation of optimal transport is to organize samples in the source domain into structures according to their class labels, and samples in the target domain into structures by clustering. This organization offers a new paradigm of perceiving each domain as a measure of measures. Rigorously, each domain can be seen as a distribution over structures, where the structures are also distributions, but over samples. Hierarchical optimal transport attempts then to align the structures of both domains while minimizing the total cost of the transportation quantified by the Wasserstein distance, which acts as the ground metric. While presenting very interesting empirical performances, it turns out that the work of \citep{el2022hierarchical} has no theoretical guarantees, despite it may be an untapped potential solution to avoid the limitations listed above, specifically the one concerning the imperceptible trade-off between the three terms of the bound.
\\
\\ \textbf{Contributions:} In this paper, we address the aforementioned limitations by providing new generalization bounds based on hierarchical optimal transport. The main underlying idea behind these bounds is to decompose the two domains into structures and then indicate explicitly which structures should be aligned together to lead to a good adaptation.
\\
\\ This paper’s contributions are threefold:
\begin{enumerate}[(i)] 
    \item We provide a theoretical analysis of the work of \citep{el2022hierarchical}, which justifies the use of hierarchical optimal transport for domain adaptation.
    \item We consider the usual hierarchical organization of data into structures and introduce a new divergence measure to quantify the similarity between source and target domains in light of this hierarchy, which we call the Hierarchical Wasserstein distance. We relate the proposed distance to the classical Wasserstein distance.
    \item We derive generalization bounds on the target risk based on the hierarchical Wasserstein distance, for the three domain adaptation settings: unsupervised, semi-supervised, and multi-source domain adaptation. The proposed generalization bounds indicate the distance between which structures should be really minimized to lead to a good adaptation. This makes the trade-off between the three terms of the bound more explicit and may suggest the minimization of each term independently of the others.
\end{enumerate}
\textbf{Outline \,} The rest of this paper is organized as follows. The $2\textsuperscript{nd}$ section provides the necessary background on hierarchical optimal transport and its application to domain adaptation. The $3\textsuperscript{rd}$ section introduces the Hierarchical Wasserstein distance as a divergence measure between the source and target domains and introduces the link with the classical Wasserstein distance. The $4\textsuperscript{th}$ section proves generalization bounds based on the Hierarchical Wasserstein distance for three scenarios, unsupervised, semi-supervised, and multi-source domain adaptation. Finally, we discuss conclusions and future research directions in section 5.
\section{Preliminary Knowledge}
In this section, we first introduce the basic notions of optimal transport \Citep{villani2009optimal} and its hierarchical formulation. Then, we show how it found application in domain adaptation.
\subsection{Optimal Transport}
The birth of optimal transport is dated back to 1781,
when the French mathematician Gaspard Monge introduced the following problem:  
\begin{definition}[The problem of Monge, \Citealt{monge1781memoire}]
    Let $(\mathcal{X},\mu)$ and $(\mathcal{Y},\nu)$ be two probability spaces, $c : \mathcal{X}\times\mathcal{Y} \to \mathbb{R}^+ $ a positive cost function over $\mathcal{X}\times\mathcal{Y}$, which represents the work needed to move a unit of mass from $x \in \mathcal{X}$ to $y \in \mathcal{Y}$. The problem of Monge asks to find a measurable transport map  $\mathrm{T} : \mathcal{X} \to \mathcal{Y}$ that transports the mass represented by the probability measure $\mu$ to the mass represented by the probability measure $\nu$, while minimizing the total cost of this transportation:
\begin{equation} (\mathcal{M}) \,\,\,\,\,\,\underset{\mathrm{T}}{\inf}\left\{\int_{\mathcal{X}}  c(x,\mathrm{T}(x)) d\mu(x) \mid \mathrm{T}\#\mu = \nu \right\}, \label{eq1}  \end{equation}
where $\mathrm{T}\#\mu$ stands for the image measure of $\mu$ by $\mathrm{T}$.
\end{definition}
The problem of Monge $(\mathcal{M})$ is quite difficult,  since it is not symmetric, and may not admit a solution, it is the case when $\mu$ is a Dirac measure and $\nu$ is not.
\\
\\ A long period of sleep followed  Monge’s formulation until the convex relaxation of the Soviet mathematician Leonid Kantorovitch in the thick of World War II, known as the problem of Monge-Kantorovich:
\begin{definition}[The problem of Monge-Kantorovich, \Citealt{kantorovich1942translocation}]
    Let $(\mathcal{X},\mu)$ and $(\mathcal{Y},\nu)$ be two probability spaces, $c : \mathcal{X}\times\mathcal{Y} \to \mathbb{R}^+ $ a positive cost function over $\mathcal{X}\times\mathcal{Y}$. The problem of Monge-Kantorovich asks to find a joint probability measure $\gamma \in \Pi(\mu,\nu)$ that minimizes:
\begin{equation} (\mathcal{MK}) \,\,\,\,\,\,\underset{\gamma}{\inf} \left\{\, \int_{\mathcal{X}\times\mathcal{Y}} \, c(x,y) \, d\gamma(x,y) \mid \gamma \in \Pi(\mu,\nu)\, \right\}, \label{eq2} \end{equation}
where $\Pi(\mu,\nu) = \{ \gamma \in \mathcal{P}(\mathcal{X}\times\mathcal{Y}) \mid proj_{\mathcal{X}}\#\gamma = \mu$, $proj_{\mathcal{Y}}\#\gamma = \nu$\} is the transport plans set, constituted of all joint probability measures $\gamma$ on $\mathcal{X}\times\mathcal{Y}$ with marginals $\mu$ and $\nu$.
\end{definition}
 The relaxed formulation $(\mathcal{MK})$ allows mass splitting and, in contrast to the formulation of Monge, it guarantees the existence of a solution under very general assumptions. When $\mathcal{X}=\mathcal{Y}$ is a polish metric space endowed with a distance $d$,  a natural choice is to use it as a cost function, e.g., $c(x, y) = d(x, y)^p$ for $p \in {\left[1\,,+\infty\right[}$. Then, the problem $(\mathcal{MK})$ induces a metric between probability measures in the Wasserstein space $\mathcal{P}_p(\mathcal{X})$, called the $p$-Wasserstein distance.
\begin{definition}[Wasserstein space, \Citealt{villani2009optimal}] \label{def:Wasserstein Space} Let $(\mathcal{X},d)$ be a Polish metric space and let $p \in  \left[ 1,\infty \right[$. The Wasserstein space of order $p$ is defined as:
\begin{equation}
    \mathcal{P}_p(\mathcal{X}) = \left\{\mu \in \mathcal{P}(\mathcal{X}) \mid \int_\mathcal{X} d(x_0,x)^p d\mu(x) < +\infty \right\},
\end{equation}
where $x_0 \in \mathcal{X}$ is arbitrary. 
\end{definition}
The space $\mathcal{P}_p(\mathcal{X})$ does not depend on the choice of the point $x_0$. In other words, the Wasserstein space $\mathcal{P}_p(\mathcal{X})$ is the space of probability measures which have a finite moment of order $p$.
\begin{definition}[Wasserstein distance, \Citealt{villani2009optimal}] \label{def:Wasserstein distance} Let $(\mathcal{X},d)$ be a Polish metric space and let $p \in  \left[ 1,\infty \right[$. For any two probability measures $\mu,\nu$ in $\mathcal{P}_p(\mathcal{X})$, the Wasserstein distance of order $p$ between $\mu$ and $\nu$ is defined by:
\begin{equation}
    \mathcal{W}_{p}(\mu,\nu) = \left(\underset{\gamma \in \Pi(\mu,\nu)}{\inf} \int_{\mathcal{X}^{2}} d(x,y)^{p} \, d\gamma(x,y)\right)^{1/p},
\end{equation}
where $\Pi(\mu,\nu) = \{\gamma \in \mathcal{P}(\mathcal{X}^{2}) \mid proj_{\mathcal{X}}\#\gamma = \mu, \, proj_{\mathcal{Y}}\#\gamma = \nu \}$ is the transport plans set, constituted of all joint probability measures $\gamma$ on $\mathcal{X}^{2}$ that have marginals $\mu$ and $\nu$.
\end{definition}
In the discrete version of Monge-Kantorovich formulation, i.e., when the measures $\mu$ and $\nu$ are only available through discrete samples $X=\{x_1,...,x_n\}\subset\mathcal{X}  \, \text{and}  \, Y = \{y_1,...,y_m\} \subset \mathcal{Y}$, their empirical distributions can be expressed as
$\mu = \sum_{i=1}^n a_{i} \delta_{x_{i}}$ and $\nu = \sum_{j=1}^m b_{j} \delta_{y_{j}}$, where  $a=(a_1,...,a_n)$ and $b=(b_1,...,b_m)$ are vectors in the probability simplex $\sum_n$ and $\sum_m$ respectively. The cost function on its part, only needs to be specified for every pair $(x_i,y_j)_{\underset{1\leq j \leq m} {1\leq i \leq n}} \in X \times Y$ yielding a cost matrix $C \in  \mathcal{M}_{n \times m}(\mathbb{R}^{+})$. The problem $(\mathcal{MK})$ becomes then a linear program \citep{bertsimas1997introduction}, as follows:
\begin{definition}[Discrete Monge-Kantorovich formulation, \Citealt{peyre2019computational}]
Let $\mu$ $= \sum_{i=1}^n a_{i} \delta_{x_{i}}$ and $\nu = \sum_{j=1}^m b_{j} \delta_{y_{j}}$ be two discrete probability measures in $\mathcal{P(X)}$, and let $C \in  \mathcal{M}_{n \times m}(\mathbb{R}^{+})$ be the cost matrix. The discrete problem of Monge-Kantorovich asks to find an optimal transport plan  $\gamma \in U(a,b)$ that minimizes: 
\begin{equation}
(\mathcal{D}_\mathcal{MK}) \,\,\,\,\,\,\underset{\gamma \in U(a,b)}{\min}  \langle {\gamma},{C} \rangle _F
\label{eq4}  \end{equation}
where $U(a,b) = \{\gamma \in \mathcal{M}_{n \times m}(\mathbb{R}^{+}) \mid \gamma \mathds{1}_{m} = a \,\, \text{and} \,\, \gamma^{\mathbf{T}} \mathds{1}_{n} = b\}$ is the transportation polytope and $\langle.,.\rangle_F$ is the Frobenius inner product.
\end{definition}
In this case, the $p$-Wasserstein distance can be inferred as follows: $\mathcal{W}_{p}^p(\mu,\nu) = \langle {\gamma ^{*}},{C} \rangle _F$. 
\\ 
\\As stated above, discrete optimal transport is a linear program, and thus can be solved exactly in $\mathcal{O}(r^3 \log(r))$ where $r= \max(n,m)$, with the simplex algorithm or interior point methods \citep{pele2009fast}, which is a heavy computational price tag. Entropy-regularization \citep{cuturi2013sinkhorn} has emerged as a solution to this computational burden:
\begin{definition}[Entropic regularization of optimal transport, \Citealt{cuturi2013sinkhorn}]
The 
\\entropy-regularized discrete optimal transport problem is defined as follows:
\begin{equation}
(\mathcal{D}_\mathcal{MK}^{\,\varepsilon}) \,\,\,\,\,\,\underset{\gamma \in U(a,b)}{\min}  \langle {\gamma},{C} \rangle _F - \varepsilon \mathcal{H}(\gamma)  
\label{eq6} \end{equation}
where  $\mathcal{H}(\gamma) = - \sum_{i=1}^n \sum_{j=1}^m \gamma_{ij} (\log(\gamma_{ij}) - 1)  $ is the entropy of $\gamma$.
\end{definition}
This entropic regularization allows a faster computation of the optimal transport plan in $\mathcal{O}(r^2/\varepsilon^3)$ \citep{10.5555/3294771.3294958} via the iterative procedure of Sinkhorn algorithm \citep{knight2008sinkhorn}.
\newpage 
\noindent Hierarchical optimal transport has appeared as an extension for probability measures in $\mathcal{P}_\text{p}(\mathcal{P_\text{p}(X)})$, where the Wasserstein distance $\mathcal{W}_p$ occurs as the ground metric of $\mathcal{P_\text{p}(X)}$:
\begin{definition}[Hierarchical Optimal Transport, \Citealt{el2022hierarchical}]
    Let $\Phi  = \\ \{\rho_1,...,\rho_h\} 
     \subset \mathcal{P_\text{p}(X)}$ and  $\Psi =\{\varrho_1,...,\varrho_l\} \subset \mathcal{P_\text{p}(X)}$ be two sets of probability measures in $\mathcal{P_\text{p}(X)}$. The empirical distributions of $\Phi$ and $\Psi$ can be expressed respectively by $\phi,\varphi \in \mathcal{P}_\text{p}(\mathcal{P_\text{p}(X)})$ as $\phi = \sum_{i=1}^h \alpha_{i} \delta_{\rho_{i}}$ and $\psi = \sum_{j=1}^l \beta_{j} \delta_{\varrho_{j}}$, where  $\alpha=(\alpha_1,...,\alpha_h)$ and $\beta=(\beta_1,...,\beta_l)$ are vectors in the probability simplex $\sum_h$ and $\sum_l$ respectively. The discrete version of hierarchical optimal transport problem between $\phi$ and $\psi$ is given by:
\begin{equation} (\mathcal{HOT}) \,\,\,\,\,\,\underset{\Gamma \in U(\alpha,\beta)}{\min}  \langle {\Gamma},{\mathcal{W}} \rangle _F
\label{eq7} \end{equation}
where $U(\alpha,\beta) = \{\Gamma \in \mathcal{M}_{h \times l}(\mathbb{R}^{+}) \mid \Gamma \mathds{1}_{l} = \alpha \,\, \text{and} \,\, \Gamma^{\mathbf{T}} \mathds{1}_{h} = \, \beta\}$ represents the transportation polytope and $\mathcal{W} = (\mathcal{W}_{p}(\rho_i,\varrho_j))_{\underset{1\leq j \leq l} {1\leq i \leq h}} \in  \mathcal{M}_{h \times l}(\mathbb{R}^{+})$ stands for the Wasserstein cost matrix.
\end{definition}
\subsection{Domain Adaptation}
Let $\mathcal{X} = \mathbb{R}^d$ be an input space and $\mathcal{Y}=\{C_1,...,C_k\}$ a discrete label space consisting of $k$ classes. Let $\mathcal{D}$ an arbitrary distribution over $\mathcal{X} \times \mathcal{Y}$ and $\mu_\mathcal{D}$ the marginal distribution of $\mathcal{D}$ over $\mathcal{X}$. Let $\mathcal{H} = \{h \mid h : \mathcal{X} \to \mathcal{Y}\}$ be a hypothesis space that consists of functions $h$ usually called hypothesis that map each element of $\mathcal{X}$ to $\mathcal{Y}$ and let consider a loss function $l: \mathcal{Y} \times \mathcal{Y} \to \left[0, 1\right]$ that gives a cost of $h(x)$ deviating from the true output $ y = f_\mathcal{D}(x) \in \mathcal{Y}$, where $f_\mathcal{D}$ is the true labeling function associated to the distribution $\mathcal{D}$. The true risk of a given hypothesis $h$ is defined as follows:
\begin{definition}[True risk] 
\label{def:true_risk}
Given a hypothesis $h \in\mathcal{H}$, a joint probability distribution $\mathcal{D}$, and a loss function $l$, the true risk of $h$ over $\mathcal{D}$ is:
\begin{equation}
 \epsilon_\mathcal{D}(h) = \underset{(x,y) \sim \mathcal{D}}{\mathbb{E}} l(h(x),y).
\end{equation}
By abuse of notations, for a given pair of hypothesis $(h,h') \in \mathcal{H}^2$, we can write:
\begin{equation}
 \epsilon_\mathcal{D}(h,h') = \underset{(x,y) \sim \mathcal{D}}{\mathbb{E}} l(h(x),h'(x)).
\end{equation}
\end{definition}  
In the context of unsupervised domain adaptation, let us consider $\mathcal{S}$ and $\mathcal{T}$ as two different probability distributions over $\mathcal{X} \times \mathcal{Y}$ called respectively the source and target domains. We have access to a set $S=\{(x_i,y_i)\}_{i=1}^n$ of $n$ labeled source samples drawn i.i.d. from the joint distribution $\mathcal{S}$ and a set $T=\{x_j\}_{j=1}^m$ of $m$ unlabeled target samples drawn i.i.d. from the marginal distribution $\mu_\mathcal{T}$, of the joint distribution $\mathcal{T}$ over $\mathcal{X}$, more formally:
\begin{equation}
    S=\{(x_i,y_i)\}_{i=1}^n \sim (\mathcal{S})^n, \quad  T=\{x_j\}_{j=1}^m \sim (\mu_\mathcal{T})^m
\end{equation}
The aim of unsupervised domain adaptation is to learn a hypothesis $h \in \mathcal{H}$ with a low target risk:
\begin{equation}
 \epsilon_\mathcal{T}(h) = \underset{(x,y) \sim \mathcal{T}}{\mathbb{E}} l(h(x),y).
\end{equation}
under the distributional shift assumption $\mathcal{S} \neq \mathcal{T}$. 
\\
\\ In the rest, we design by the source domain interchangeably the distribution $\mathcal{S}$ and the labeled set $S$, and by the target domain, the distribution $\mathcal{T}$ and the unlabeled set $T$.
\subsection{Hierarchical Optimal Transport for Unsupervised Domain Adaptation}
For the first time, the hierarchical formulation of optimal transport was applied to unsupervised domain adaptation in \citep{el2022hierarchical} as follows:
Labeled samples in the source domain $S=\{(x_i,y_i)\}_{i=1}^n$ can be grouped into $k$ classes $\{C_1,...,C_k\}$ according to their class labels. Similarly, unlabeled samples in the target domain $T=\{x_j\}_{j=1}^m$ can be grouped into $k$ clusters $\{Cl_1,...,Cl_k\}$. The empirical distribution of each class $C_h$ and cluster $Cl_l$ can be expressed  respectively using a discrete measures $\rho_h$ and $\varrho_l$ in $\mathcal{P}_\text{p}(\mathcal{X})$ as follows:
\begin{equation}
    \rho_h = \sum_{i=1/ x_i \in C_h}^n \frac{1}{\lvert C_h \rvert} \delta_{x_i}\quad \text{and} \quad  \varrho_l = \sum_{j=1/ x_j \in Cl_l}^m \frac{1}{\lvert Cl_l \rvert} \delta_{x_j},  \quad \forall h,l \in \{1,...,k\}
\label{eq15} \end{equation}
The set $S$ of labeled source samples and the set $T$ of unlabeled target samples can be seen in a hierarchical paradigm as a collection of classes and clusters. Thus, the marginal distribution of the source  and target domains can be expressed respectively as a measure of measures $\varphi_\mathcal{S}$ and $\varphi_\mathcal{T}$ in $\mathcal{P}_\text{p}(\mathcal{P}_\text{p}(\mathcal{X}))$ as follows:
\begin{equation}
    \varphi_\mathcal{S} = \sum_{h=1}^k \frac{1}{k}\delta_{\rho_{h}} \quad \text{and} \quad \varphi_\mathcal{T} = \sum_{l=1}^k \frac{1}{k} \delta_{\varrho_{l}}
\label{eq17} \end{equation}
To learn the correspondence between classes and clusters, we formulate an entropy-regularized hierarchical optimal transport problem between $\varphi_\mathcal{S}$ and $\varphi_\mathcal{T}$ in the following way:
\begin{equation} (\mathcal{HOT\textbf{-}DA}) \,\,\,\,\,\,\underset{\Gamma \in U(\alpha,\beta)}{\min}  \langle {\Gamma},{\mathcal{W}} \rangle _F - \varepsilon \mathcal{H}(\Gamma) 
\label{eq18} \end{equation}
where $\alpha=(\frac{1}{k},...,\frac{1}{k})$ and $\beta=(\frac{1}{k},...,\frac{1}{k})$ are vectors in the probability simplex $\sum_k$, $U(\alpha,\beta) = \{\Gamma \in \mathcal{M}_{k}(\mathbb{R}^{+}) \mid \Gamma \mathds{1}_{k} = \alpha \,\,\, \text{and} \,\,\, \Gamma^{\mathbf{T}} \mathds{1}_{k} = \beta\}$ represents the transportation polytope and  $\mathcal{W} = (\mathcal{W}_{h,l})_{{1\leq h,l \leq k}} \in  \mathcal{M}_{k}(\mathbb{R}^{+})$ stands for the Wasserstein cost matrix, whose each matrix-entry $\mathcal{W}_{h,l}$ is defined as the $2$-Wasserstein distance between the measures $\rho_h$ and $\varrho_l$, i.e., $\mathcal{W}_{h,l} = \mathcal{W}_{2}(\rho_h,\varrho_l)$.
\\ 
\\The optimal transport plan $\Gamma^*_\varepsilon$ in \eqref{eq18} can be interpreted as a soft multivalued matching between $\varphi_\mathcal{S}$ and $\varphi_\mathcal{T}$ as it provides the degree of association between classes $\{C_1,...,C_k\}$ in the source domain $S$ and clusters $\{Cl_1,...,Cl_k\}$ in the target domain $T$. Thus, a one-to-one matching relationship $(\widehat{=})$ between each class $C_h$ and its corresponding cluster $Cl_l$ can be inferred by hard assignment from $\Gamma^*_\varepsilon$, in the following way:
\begin{equation}
    C_h  \widehat{=} \, Cl_l \mid l = \underset{j=1,...,k}{\text{argmax}} \, \Gamma^*_\varepsilon(h,j), \quad \forall h \in \{1,...,k\}
\label{eq20} \end{equation}
Once the correspondence between source and target structures has been determined according to the one-to-one matching relationship $(\widehat{=})$ in \eqref{eq20}, the source samples in each class $C_h$ have to be transported to the target samples in the corresponding cluster $Cl_l$ via the barycentric mapping that can be formulated for each class $C_h$ as follows:
\begin{equation}
    \widetilde{C_h} = diag(\gamma^{*,\varepsilon_{'}}_{h,l}\mathds{1}_{\lvert Cl_l \rvert})^{-1}\gamma^{*,\varepsilon_{'}}_{h,l}Cl_l, \quad \forall h \in \{1,...,k\}
\label{eq22} \end{equation}
After the alignment of each class $C_h$ with its corresponding cluster $Cl_l$ has been done as suggested in \eqref{eq22}, a classifier $h$ can be learned on the transported labeled source data $\widetilde{S}= \cup_{h=1}^k \, \widetilde{C_h}$ and evaluated on the unlabeled target data $T$. 
\newpage
\section{Hierarchical Wasserstein distance}
\setlength{\parindent}{0em}
This section is dedicated to constructing the Hierarchical Wasserstein distance $\mathcal{HW}_p$ on the space $\mathcal{P}_p(\mathcal{P}_p(\mathcal{X}))$ that will serve as a divergence measure to quantify the closeness between the source and target domains.
\\
\\ The Hierarchical Wasserstein distance will allow us to introduce several generalization bounds on the target risk in the next section where we will study three scenarios of domain adaptation as we will see later.
\\
\\ First, let us present the following Theorem from \citep{villani2009optimal} that will be useful for establishing the $\mathcal{HW}_p$ distance.
\begin{theorem}[Topology of the Wasserstein space] \label{thr:Topology of the Wasserstein space} 
Let $(\mathcal{X},d)$ be a Polish metric space and let $p \in  \left[1,\infty\right[$. Then the Wasserstein space $\mathcal{P}_p(\mathcal{X)}$ metrized by the Wasserstein distance $\mathcal{W}_{p}$ is itself a Polish metric space.
\end{theorem}
In the following Lemma we show that $\mathcal{HW}_{p}$ is effectively a distance on the space $\mathcal{P}_p(\mathcal{P}_p(\mathcal{X}))$.
\begin{lemma} [Hierarchical Wasserstein distance] \label{Hierarchical Wasserstein distance} 
Let $(\mathcal{X},d)$ be a Polish metric space and let $p \in  \left[ 1,\infty\right[$. For any two probability measures $\phi,\varphi \in \mathcal{P}_p(\mathcal{P}_p(\mathcal{X}))$, the Hierarchical Wasserstein distance of order $p$ between $\phi$ and $\varphi$ is defined by:
\begin{equation}
    \mathcal{HW}_{p}(\phi,\varphi) = \left(\underset{\eta \in \Pi(\varphi,\varphi)}{\inf} \int_{\mathcal{P}_p(\mathcal{X})^{2}} \mathcal{W}_{p}(\rho,\varrho)^{p} \, d\eta(\rho,\varrho)\right)^{1/p}.
\end{equation}
Furthermore $\mathcal{P}_p(\mathcal{P}_p(\mathcal{X}))$ metrized by $\mathcal{HW}_{p}$ is a Polish metric space.
\end{lemma}
\begin{proof} Let $p \in  \left[ 1,\infty\right[$. Since $(\mathcal{X},d)$ is a Polish metric space, then $(\mathcal{P}_p(\mathcal{X}),\mathcal{W}_{p})$ is itself a Polish metric space by Theorem \ref{thr:Topology of the Wasserstein space}.
By a recursion of concepts, Definition \ref{def:Wasserstein distance} ensures that $\mathcal{HW}_p$ defines a distance on the space  $\mathcal{P}_p(\mathcal{P}_p(\mathcal{X}))$ and Theorem \ref{thr:Topology of the Wasserstein space} holds that $(\mathcal{P}_p(\mathcal{P}_p(\mathcal{X})),\mathcal{HW}_{p})$ is a Polish metric space. 
\end{proof}
The following Corollary from \citep{villani2009optimal} is of particular interest for the statement of the link between $\mathcal{W}_p$ and $\mathcal{HW}_p$.
\begin{corollary} [Measurable selection of optimal plans]  \label{Measurable selection of optimal plans}  Let $\mathcal{X},\mathcal{Y}$ be  Polish spaces and  let $c : \mathcal{X} \times \mathcal{Y} \to \mathbb{R}$ be a continuous cost function, inf $c > - \infty$. Let $\Upsilon$ be a measurable space and let $\upsilon \mapsto (\mu_\upsilon,\nu_\upsilon)$ be a measurable function $\Upsilon \to \mathcal{P}(\mathcal{X}) \times \mathcal{P}(\mathcal{Y})$. Then there is a measurable choice $\upsilon \mapsto \pi_\upsilon$ such that for each $\upsilon$, $\pi_\upsilon$ is an optimal transport plan between $\mu_\upsilon$ and $\nu_\upsilon$.
\end{corollary} 
In the following Lemma, we prove that the Wasserstein distance enjoys an interesting property that we call hierarchical monotonicity.
\begin{lemma}[Hierarchical monotonicity of Wasserstein distance] \label{Hierarchical monotonicity of Wasserstein distance}
Let $(\mathcal{X},d)$ be a Polish metric space and let $p \in  \left[ 1,\infty\right[$. Let $\phi,\varphi \in \mathcal{P}_p(\mathcal{P}_p(\mathcal{X}))$ and let  $\mu,\nu \in \mathcal{P}_p(\mathcal{X})$ such that $ \mu = \int_{\mathcal{P}_p(\mathcal{X})} Xd\phi$ and $ \nu = \int_{\mathcal{P}_p(\mathcal{X})} Xd\varphi$ for some generic measure-valued random variable $X$. The following holds,
\begin{equation}
    \mathcal{W}_p (\mu,\nu) \leq  \mathcal{HW}_p (\phi,\varphi).
\end{equation}
\end{lemma}
\newpage
\begin{proof}
Let $\phi,\varphi \in \mathcal{P}_p(\mathcal{P}_p(\mathcal{X}))$ and let consider an arbitrary $\eta \in \Pi(\phi,\varphi)$, then:
\begin{align}
\int_{\mathcal{P}_p(\mathcal{X})^2} \mathcal{W}_p^p  (\rho,\varrho) d\eta (\rho,\varrho) &= \int_{\mathcal{P}_p(\mathcal{X})^2} \bigg( \int_{\mathcal{X}^2}  d(x,y)^p \pi_{\rho,\varrho}(dx,dy) \bigg) d\eta (\rho,\varrho)
\\ &= \int_{\mathcal{X}^2}  \bigg( \int_{\mathcal{P}_p(\mathcal{X})^2} d(x,y)^p \pi_{\rho,\varrho} d\eta(\rho,\varrho)  \bigg )   (dx,dy) 
\\ &= \int_{\mathcal{X}^2} d(x,y)^p \bigg( \int_{\mathcal{P}_p(\mathcal{X})^2}  \pi_{\rho,\varrho} d\eta(\rho,\varrho)  \bigg )   (dx,dy) 
\\ &= \int_{\mathcal{X}^2}  d(x,y)^p \pi(dx,dy) 
\\ &\geq \mathcal{W}_p^p(\mu,\nu) = \underset{\pi\in \Pi(\mu,\nu)}{\inf} \int_{\mathcal{X}^{2}} d(x,y)^{p} \, \pi(dx,dy)
\end{align}
First line is obtained by the definition of the Wasserstein distance and by using the measurable selection of optimal plans, so $\pi_{\rho,\varrho}$ is an optimal transport plan between $\rho$ and $\varrho$ that is chosen in a measurable way according to Corollary \ref{Measurable selection of optimal plans}. Second line is due to Fubini’s theorem. Third line is trivial. Fourth line follows from the fact that $\int_{\mathcal{P}_p(\mathcal{X})^2} \pi_{\rho,\varrho} d\eta(\rho,\varrho) = \pi$ for some valid transport plan $\pi \in \Pi(\mu,\nu)$, this becomes clear by marginalizing out $y$ and marginalizing out $x$, respectively:
\begin{align}
\forall \mathcal{A} \subset \mathcal{X}: \int_{\mathcal{P}_p(\mathcal{X})^2} \pi_{\rho,\varrho}(\mathcal{A}\times\mathcal{X}) d\eta(\rho,\varrho) &= \int_{\mathcal{P}_p(\mathcal{X})^2} \rho (\mathcal{A}) d\eta(\rho,\varrho)  \label{eqa.1}
\\ &= \int_{\mathcal{P}_p(\mathcal{X})} \rho (\mathcal{A}) d\phi \label{eqa.2}
\\ &= \mu(\mathcal{A}) \label{eqa.3}\\
\forall \mathcal{B} \subset \mathcal{X}: 
 \int_{\mathcal{P}_p(\mathcal{X})^2} \pi_{\rho,\varrho}(\mathcal{X}\times\mathcal{B}) d\eta(\rho,\varrho)  &= \int_{\mathcal{P}_p(\mathcal{X})^2} \varrho (\mathcal{B}) d\eta(\rho,\varrho)  \label{eqb.1}
\\ &= \int_{\mathcal{P}_p(\mathcal{X})} \varrho (\mathcal{B}) d\varphi   \label{eqb.2}
\\ &= \nu(\mathcal{B})   \label{eqb.3}
\end{align}
The first equalities \eqref{eqa.1} and \eqref{eqb.1} follow from the fact that $\pi_{\rho,\varrho}$ is an optimal  transport plan between $\rho$ and $\varrho$. Second equalities \eqref{eqa.2} and \eqref{eqb.2} follow from the fact that $\eta$ is an optimal transport plan between $\phi$ and $\varphi$. Third equalities \eqref{eqa.3} and \eqref{eqb.3} follow from the assumptions made on $\phi$ and $\varphi$, respectively. 
\\
\\ Let's get back to the core of the proof, inequality in the fifth line follows from the definition of the Wasserstein distance. 
\\
\\ The inequality $\int_{\mathcal{P}_p(\mathcal{X})^2} \mathcal{W}_p^p  (\rho,\varrho) d\eta (\rho,\varrho)  \geq \mathcal{W}_p^p(\mu,\nu)$ holds for any $\eta \in \Pi(\phi,\varphi)$, then, we obtain the final result by taking the infimum over $\eta$ from the left-hand side, i.e.
\begin{align}
\underset{\eta\in \Pi(\phi,\varphi)}{\inf} \int_{{\mathcal{P}_p(\mathcal{X})}^{2}} \mathcal{W}_p^p  (\rho,\varrho) d\eta (\rho,\varrho) &\geq \mathcal{W}_p^p(\mu,\nu)
\end{align}
which gives: 
\begin{align}
\mathcal{HW}_p (\phi,\varphi) &\geq \mathcal{W}_p(\mu,\nu)  
\end{align}
\end{proof}
\vspace{-30.pt}
\section{Generalization bounds based on the Hierarchical Wasserstein distance}
In this section, we introduce generalization bounds on the target risk when the divergence between the source and target domains is measured by the Hierarchical Wasserstein distance.
\subsection{A bound for unsupervised domain adaptation}
This subsection focuses on unsupervised domain adaptation where no labeled data are available in the target domain. We first present the Lemma that introduces Hierarchical Wasserstein distance to relate the source and target risks for an arbitrary pair of hypothesis.
\begin{lemma}  \label{Mourad}
Let $\mu_\mathcal{S},\mu_\mathcal{T} \in \mathcal{P}_p(\mathcal{X})$ be two probability measures on a compact $\mathcal{X} \subseteq \mathbb{R}^\textup{d}$ and let $\varphi_\mathcal{S},\varphi_\mathcal{T} \in \mathcal{P}_p(\mathcal{P}_p(\mathcal{X}))$ be two probability measures on $\mathcal{P}_p(\mathcal{X})$ such that $ \mu_\mathcal{S} = \int_{\mathcal{P}_p(\mathcal{X})} Xd\varphi_\mathcal{S}$ and $ \mu_\mathcal{T} = \int_{\mathcal{P}_p(\mathcal{X})} Xd\varphi_\mathcal{T}$ for some generic measure-valued random variable $X$. Assume that the cost function $c(x,y) = \Vert \phi(x) - \phi(y) \Vert_{\mathcal{H}_{k}}$, where $\mathcal{H}_{k}$ is a reproducing kernel Hilbert space (RKHS) equipped with kernel $k: \mathcal{X} \times \mathcal{X} \rightarrow \mathbb{R}$ induced by $\phi: \mathcal{X} \rightarrow \mathcal{H}_{k}$ and $k(x,y) = \langle \phi(x), \phi(y) \rangle_{\mathcal{H}_{k}}$. Assume further that the loss function $l_{h,f}:x \mapsto l(h(x),f(x))$ is convex, symmetric, bounded, obeys triangle equality, and has the parametric form $\vert h(x) - f(x) \vert^q$ for some $q > 0$. Assume also that the kernel $k$ in the RKHS $\mathcal{H}_{k}$ is square-root integrable {\it w.r.t.} both $\mu_\mathcal{S},\mu_\mathcal{T} $ for all  $\mu_\mathcal{S},\mu_\mathcal{T} \in \mathcal{P}_p(\mathcal{X})$ where $0\leq k(x,y) \leq K, \forall x,y \in \mathcal{X}$. If $\Vert l \Vert_{\mathcal{H}_{k}}\leq 1$, then the following holds:
\begin{align}
\forall (h,h')\in \mathcal{H}_{k}^2,\quad  \epsilon_\mathcal{T}(h,h') \leq \epsilon_\mathcal{S}(h,h')+ \mathcal{HW}_1(\varphi_\mathcal{S},\varphi_\mathcal{T}).
\end{align}
\end{lemma}

\begin{proof}
Under assumptions of Lemma \ref{Mourad}, and according to \citep{redko2017theoretical}, we have: 
\begin{equation}
    \forall (h,h')\in \mathcal{H}_{k}^2,\quad  \epsilon_\mathcal{T}(h,h') \leq \epsilon_\mathcal{S}(h,h')+ \mathcal{W}_1(\mu_\mathcal{S},\mu_\mathcal{T})
\end{equation}
On the other hand, using the property of Hierarchical monotonicity of Wasserstein distance in Lemma \ref{Hierarchical monotonicity of Wasserstein distance} for $p=1$, we have:
\begin{equation}
    \mathcal{W}_1(\mu_\mathcal{S},\mu_\mathcal{T}) \leq \mathcal{HW}_1(\varphi_\mathcal{S},\varphi_\mathcal{T})
\end{equation}
which gives: 
\begin{equation}
     \forall (h,h')\in \mathcal{H}_{k}^2,\quad  \epsilon_\mathcal{T}(h,h') \leq \epsilon_\mathcal{S}(h,h')+ \mathcal{HW}_1(\varphi_\mathcal{S},\varphi_\mathcal{T})
\end{equation}
\end{proof}
\vspace{-20.pt}
\begin{remark}
Lemma \ref{Mourad} and the subsequent results are established for the special case $p=1$, but they can easily be generalized for any $ p > 1$, by applying Hölder's inequality that states: 
\begin{align}
    p \leq q \Rightarrow \mathcal{HW}_p \leq \mathcal{HW}_q
\end{align}
\end{remark}
\begin{remark}
As reported in \citep{redko2017theoretical}, the parametric form of the loss function $l_{h,f}$ as $\vert h(x) - f(x) \vert^q$ for some $q > 0$ is only an example. Following \citep{saitoh1997integral}, we can also look at more general nonlinear transformations of $h$ and $f$ that satisfy the hypothesis made on $l_{h,f}$ above. These transformations can comprise a product of hypothesis and labeling functions and thus the suggested results are relevant for hinge loss too.
\end{remark}
\begin{remark}
Lemma \ref{Mourad} supposes that the cost function $c(x,y) = \Vert \phi(x) - \phi(y) \Vert_{\mathcal{H}_{k}}$. This may seem too demanding as in several applications, the Euclidean distance $c(x,y) = \Vert x-y \Vert$ is considered as the ground metric. But fortunately, this assumption is not that restrictive and may be bypassed through the duality between RKHS and distance-based metric representations, studied by \citep{sejdinovic2013equivalence}. In fact:
\begin{equation} \small
    \Vert \phi(x) - \phi(y) \Vert_{\mathcal{H}_{k}} = \sqrt{\langle \phi(x) - \phi(y) ,\phi(x) - \phi(y) \rangle}_{\mathcal{H}_{k}} = \sqrt{k(x,x) - 2k(x,y) + k(y,y)}.
\end{equation}
Thus, the Euclidean distance can be recovered by considering the kernel provided by the covariance function of the fractional Brownian motion: 
\begin{equation}
    k(x,y) = \frac{1}{2}\left(\Vert x \Vert^2 - \Vert x -y \Vert^2 + \Vert y \Vert^2\right)
\end{equation}
\end{remark}
We report now some preliminary results to show the convergence of an empirical measure to its true associated measure with respect to the Wasserstein distance. These results can be extended to the Hierarchical Wasserstein distance, which allows to provide generalization bounds for finite samples rather than true population measures. First, let's define Talagrand
inequalities $T_p$ as in \citep{villani2009optimal}.
\begin{definition} [$T_p$ inequality] \label{Concentration inequality} 
Let $(\mathcal{X},d)$ be a Polish metric space and let $p \in  \left[1,\infty\right[$. Let $\nu$ be a reference probability measure in $\mathcal{P}_p(\mathcal{X})$ and let $\zeta > 0$. It is said that $\nu$ satisfies $T_p(\zeta)$ inequality if:
\begin{equation}
  \forall \mu \in \mathcal{P}_p(\mathcal{X}) \quad   \mathcal{W}_p(\nu,\mu) \leq \sqrt{\frac{2 H(\nu|\mu)}{\zeta}}
\end{equation}
where $H$ is the relative entropy: $H(\nu|\mu) = \int \frac{d\nu}{d\mu} \log \frac{d\nu}{d\mu} d\mu $.
\\
\\ We shall say that $\nu$ satisfies a $T_p$ inequality if it satisfies $T_p(\zeta)$ for some constant $\zeta > 0$.
\end{definition}
Probability measures verifying $T_1$ inequality have a characteristic property related to the existence of a square-exponential moment, as shown in \citep{bolley2005weighted}.
\begin{theorem} [Characteristic property of $T_1$ inequality] \label{Explicit condition of $T_p$ inequality}
Let $\mathcal{X}$ be a measurable space equipped with a measurable distance $d$, let $\nu$ be a reference probability measure on $\mathcal{X}$, and let $x_0$
be any element of $\mathcal{X}$. Then $\nu$ satisfies $T_1$ inequality if and only if, for some $\alpha > 0$:
\begin{equation}
    \int_\mathcal{X} e^{\alpha d(x_0,x)^2} d\nu(x) < +\infty, 
\end{equation}
\end{theorem}
In \citep{bolley2007quantitative}, authors assume a $T_p$ inequality for the measure $\mu$, and derive an upper bound in $\mathcal{W}_p$ distance, we present here the case $p=1$.
\begin{theorem}[Upper bound in $\mathcal{W}_1$] \label{Concentration inequality22} 
Let $(\mathcal{X},d)$ be a Polish metric space. Let $\mu$ be a probability measure on $\mathcal{X}$ so that for some $\alpha > 0$, we have for any $x_0 \in \mathcal{X}$, $\int_\mathcal{X} e^{\alpha d(x_0,x)^2} d\mu(x) < +\infty$, and let $\hat \mu= \frac{1}{n}\sum_{i=1}^{n} \delta_{x_i}$ be its associated  empirical measure defined on a sample of independent variables $\{x_i\}_{i=1}^{n}$ all distributed
according to $\mu$. Then for any $\textup{d}'>dim(\mathcal{X})$ and $\zeta' < \zeta$, there exists some constant $N_0$ depending on $\textup{d}',\zeta'$ and some square exponential moment of $\mu$, such that for any $\varepsilon > 0$ and $N \geq N_0 \max(\varepsilon^{-(\textup{d}'+2)},1)$
\begin{align}
\mathbb{P}\left[\mathcal{W}_1(\mu,\hat \mu) > \varepsilon\right] \leq  \exp\left(\frac{-\zeta'}{2}N\varepsilon^2\right).
\end{align}
\end{theorem}
\begin{remark}
    The original version of Theorem \ref{Concentration inequality22} is established for $\mathcal{X}=\mathbb{R}^\textup{d}$, but we can find the generalization above for any metric space $(\mathcal{X},d)$ in \citep{courty2017joint}.
\end{remark}
Using Lemma \ref{Mourad} and Theorem \ref{Concentration inequality22}, we are now ready to give a generalization bound on the target risk in terms of the Hierarchical Wasserstein distance we have constructed. 
\begin{theorem} \label{TheoremMourad}
Under the assumptions of Lemma \ref{Mourad}, let $\varphi_\mathcal{S},\varphi_\mathcal{T} \in \mathcal{P}_p(\mathcal{P}_p(\mathcal{X}))$ satisfying a $T_1(\zeta)$ inequality and let $\mu_\mathcal{S},\mu_\mathcal{T} \in \mathcal{P}_p(\mathcal{X})$ such that $ \mu_\mathcal{S} = \int_{\mathcal{P}_p(\mathcal{X})} Xd\varphi_\mathcal{S}$ and $ \mu_\mathcal{T} = \int_{\mathcal{P}_p(\mathcal{X})} Xd\varphi_\mathcal{T}$. Let $S$ and $T$ be two sets of size $n$ and $m$ drawn {\it i.i.d.} from $\mu_\mathcal{S}$ and $\mu_\mathcal{T}$ respectively and let $\hat{\mu}_\mathcal{S} = \frac{1}{n}\sum_{i=1}^{n} \delta_{x_i}$ and $\hat{\mu}_\mathcal{T} = \frac{1}{m}\sum_{j=1}^{m} \delta_{x_j}$ be their associated empirical measures. 
Assume further that samples in $S$ and $T$ are grouped respectively in $k$ classes and $k$ clusters, such that, the empirical measures of $\varphi_\mathcal{S}$ and $\varphi_\mathcal{T}$ can be expressed as $\hat \varphi_\mathcal{S} = \sum_{h=1}^{k} \frac{1}{k} \delta_{\rho_h}$ and $\hat \varphi_\mathcal{T} = \sum_{l=1}^{k} \frac{1}{k} \delta_{\varrho_l}$, where $\rho_h = \sum_{i=1/ x_i \in C_h}^{n} \frac{1}{|C_h |} \delta_{x_i}$ and $ \varrho_l = \sum_{j=1/ x_j \in Cl_l}^{m} \frac{1}{|Cl_l|} \delta_{x_j}$ are the empirical measure of the $h^e$ class $C_h$  and $l^e$ cluster $Cl_l$ respectively. Then for any $\textup{d}'>dim(\mathcal{P}_p(\mathcal{X}))$ and $\zeta' < \zeta$, there exists some constant $k_0$ depending on $\textup{d}'$, such that for any $\delta > 0$ and $k \geq k_0 \max(\delta^{-(\textup{d}'+2)},1)$ with probability of at least $1-\delta$ for all $h$, the following holds:
\begin{align} \label{boundMourad}
\epsilon_\mathcal{T}(h)\leq \epsilon_\mathcal{S}(h) + \mathcal{HW}_1(\hat{\varphi}_\mathcal{S},\hat{\varphi}_\mathcal{T}) + 2\sqrt{\frac{2\log\left(\frac{1}{\delta}\right)}{\zeta'k}}+ \lambda \,,
\end{align}
where $\lambda$ is the combined error of the ideal joint hypothesis $h^*$ that minimizes the combined error of $\epsilon_\mathcal{S}(h)+\epsilon_\mathcal{T}(h)$.
\end{theorem}
\begin{proof}
    \begin{align}
\epsilon_\mathcal{T}(h) & \leq \epsilon_\mathcal{T}(h,h^*) +\epsilon_\mathcal{T}(h^*,f_\mathcal{T})
\\ &= \epsilon_\mathcal{T}(h,h^*) +\epsilon_\mathcal{T}(h^*,f_\mathcal{T}) + \epsilon_\mathcal{S}(h,h^*) - \epsilon_\mathcal{S}(h,h^*)
\\ &\leq \epsilon_\mathcal{T}(h,h^*) +\epsilon_\mathcal{T}(h^*) + \epsilon_\mathcal{S}(h) + \epsilon_\mathcal{S}(h^*)  - \epsilon_\mathcal{S}(h,h^*)
\\ &\leq \epsilon_\mathcal{S}(h) + \mathcal{HW}_1(\varphi_\mathcal{S},\varphi_\mathcal{T}) + \epsilon_\mathcal{S}(h^*) + \epsilon_\mathcal{T}(h^*) 
\\ &\leq \epsilon_\mathcal{S}(h) + \mathcal{HW}_1(\varphi_\mathcal{S},\varphi_\mathcal{T}) + \lambda 
\\ &\leq \epsilon_\mathcal{S}(h) + \mathcal{HW}_1(\varphi_\mathcal{S},\hat \varphi_\mathcal{S}) + \mathcal{HW}_1(\hat \varphi_\mathcal{S},\varphi_\mathcal{T}) + \lambda
\\ &\leq \epsilon_\mathcal{S}(h) + \sqrt{\frac{2\log\left(\frac{1}{\delta}\right)}{\zeta'k}} + \mathcal{HW}_1(\hat \varphi_\mathcal{S},\hat \varphi_\mathcal{T}) + \mathcal{HW}_1(\hat \varphi_\mathcal{T},\varphi_\mathcal{T}) + \lambda
\\ &\leq  \epsilon_\mathcal{S}(h) + \mathcal{HW}_1(\hat{\varphi}_\mathcal{S},\hat{\varphi}_\mathcal{T}) + 2\sqrt{\frac{2\log\left(\frac{1}{\delta}\right)}{\zeta'k}}+ \lambda 
\end{align}
First and third lines are obtained using the triangular inequality applied to the error
function. Fourth line is a consequence of Lemma \ref{Mourad}. Fifth line follows from the
definition of $\lambda$, sixth, seventh and eighth lines use the fact that Hierarchical Wasserstein metric is a proper distance and the Theorem \ref{Concentration inequality22} for $\mathcal{HW}_1$ applied to $\varphi_\mathcal{S}$ and $\varphi_\mathcal{T}$. 
\end{proof}
A straightforward implication of this theorem is that it justifies the application of hierarchical optimal transport in unsupervised domain adaptation. A similar result is the bound in \citep{courty2017joint}, where authors use the Wasserstein distance to measure the similarity between the joint distribution of the source domain and an estimated joint distribution of the target one. Even if this bound does not have the generic form, it suggests the minimization of the Wasserstein distance between the joint distributions, which is very close to the minimization of the Hierarchical Wasserstein distance between classes and clusters in our bound.
\\
\\ Other similar results can be found in \citep{redko2017theoretical} and  \citep{shen2018wasserstein}. The only distinction is the use of the Wasserstein distance in these bounds to measure the similarity between the marginal distributions of both domains rather than the Hierarchical Wasserstein distance in our case. Although one might think, due to the inequality in Lemma \ref{Hierarchical monotonicity of Wasserstein distance} that the proposed bound is less tight than the one in \citep{redko2017theoretical}, but our bound has a major advantage, as shown below.
\\
\\Indeed, the following Corollary gives a more explicit bound based on the development of the $\mathcal{HW}_1$ distance.
\begin{corollary} 
\label{CorollaryMourad}
Under the assumptions of Theorem \ref{TheoremMourad}, let $\Gamma^{*} =\underset{\Gamma \in U(\alpha,\beta)}{argmin}  \langle {\Gamma},{\mathcal{W}_1} \rangle _F$ be the optimal transport plan between $\hat \varphi_\mathcal{S}$ and $\hat \varphi_\mathcal{T}$, with probability of at least $1-\delta$ for all $h$, we have:
\begin{align}
\epsilon_\mathcal{T}(h)\leq \epsilon_\mathcal{S}(h) + \sum_{h=1}^{k}  \mathcal{W}_1( \rho_h, \varrho_{\sigma(h)}) + k(k-1)\iota + 2\sqrt{\frac{2\log\left(\frac{1}{\delta}\right)}{\zeta'k}}+ \lambda \,,
\end{align}
where $\begin{aligned}[t]
\sigma \colon  \{1, ...,k\} &\to  \{1, ...,k\} \\
h &\mapsto l^{*}=\underset{l}{\text{argmax}}  \, \Gamma_{h,l}^{*} 
\end{aligned}$ \quad \quad and \quad \quad $\iota=\underset{h,l\neq{\sigma(h)}}{\max} \mathcal{W}_1( \rho_h, \varrho_l)$.
\end{corollary}

\begin{proof} 
\begin{align}
 \mathcal{HW}_1(\hat \varphi_\mathcal{S},\hat \varphi_\mathcal{T})  &= \sum_{h=1}^{k}  \sum_{l=1}^{k} \mathcal{W}_1( \rho_h, \varrho_l) \Gamma_{h,l}^*
\\ &\leq \sum_{h=1}^{k}  \sum_{l=1}^{k} \mathcal{W}_1( \rho_h, \varrho_l)
\\ &= \sum_{h=1}^{k}  \mathcal{W}_1( \rho_h, \varrho_{\sigma(h)})+ \sum_{h=1}^{k} \sum_{l=1/l\neq\sigma(h)}^{k}  \mathcal{W}_1( \rho_h, \varrho_l)
\\ &\leq \sum_{h=1}^{k}  \mathcal{W}_1(\rho_h, \varrho_{\sigma(h)})+ k(k-1)\iota
\end{align}
First line follows from the definition of the Hierarchical Wasserstein distance. Second line uses the fact that $\Gamma^* \in U(\alpha,\beta)$, then we can bound each $\Gamma^*_{h,l}$ by 1 for simplicity \footnote{A tighter bound can be obtained by bounding each $\Gamma^*_{h,l}$ by $\frac{1}{k}$.}. Third and fourth lines are trivial. 
\end{proof} 
The work of \citep{el2022hierarchical} is based on the minimization of the Wasserstein distance between each class and its corresponding cluster, i.e. $\sum_{h=1}^{k} \mathcal{W}_1( \rho_h, \varrho_{\sigma(h)})$.
The minimization of this amount leads eventually to the minimization of the Hierarchical Wasserstein distance in \eqref{boundMourad}. 
\\
\\ But also, when it is accompanied by a high-quality clustering in the target domain, it leads to the transportation of labeled source data of each class together without splitting to the region occupied by the target data having the same class label \footnote{Evidently, the class labels are unknown in the target domain.}. In this sense, the algorithmic solution suggested in \citep{el2022hierarchical} in order to preserve compact classes during the transportation is explicitly reflected by the generalization bound \eqref{boundMourad}, unlike the other bounds by \citep{redko2017theoretical,shen2018wasserstein}.
\\
\\ Furthermore, this may suggest that one can independently minimize the other terms $\epsilon_\mathcal{S}(h)$ and $\lambda$ since there is no longer the concern of transporting source data of different labels to the same target data.
\subsection{A bound for semi-supervised domain adaptation}
In semi-supervised domain adaptation, when we have access to an additional small set of labeled instances $\vartheta n$ drawn independently from $\mu_\mathcal{T}$ in conjunction with $(1 - \vartheta)n$ instances drawn independently from $\mu_\mathcal{S}$ and labeled by $f_\mathcal{T}$ and $f_\mathcal{S}$, respectively. The minimization of the target risk may not be the best choice, especially if $\vartheta$ is small, which is usually the case in semi-supervised domain adaptation. Instead, we can minimize a convex combination of the empirical source and target risk, defined as follows:
\begin{align}
    \hat{\epsilon}_\theta(h) =  \theta \hat{\epsilon}_\mathcal{T}(h) + (1-\theta)\hat{\epsilon}_\mathcal{S}(h)
\end{align}
where $\theta \in \left[0,1\right].$
\\
\\ In this section, we bound the target risk of a hypothesis that minimizes $\hat{\epsilon}_\theta(h)$. The proof of the bound has two main parts, which we state as Lemmas below. 
\begin{lemma} \label{Mourad1}
Under the assumptions of Lemma \ref{Mourad}, let $\mu_\mathcal{S},\mu_\mathcal{T} \in \mathcal{P}_p(\mathcal{X})$ and let $\varphi_\mathcal{S},\varphi_\mathcal{T} \in \mathcal{P}_p(\mathcal{P}_p(\mathcal{X}))$ such that $ \mu_\mathcal{S} = \int_{\mathcal{P}_p(\mathcal{X})} Xd\varphi_\mathcal{S}$ and $ \mu_\mathcal{T} = \int_{\mathcal{P}_p(\mathcal{X})} Xd\varphi_\mathcal{T}$, let $D$  be a labeled
sample of size $n$ with $\vartheta n$ points drawn from $\mu_\mathcal{T}$ and $(1-\vartheta)n$ from $\mu_\mathcal{S}$ with $\vartheta \in (0,1)$,
and labeled according to $f_\mathcal{S}$ and $f_\mathcal{T}$. Then 
\begin{align}
\mid \epsilon_\theta(h) - \epsilon_\mathcal{T}(h) \mid \leq (1-\theta)(\mathcal{HW}_1(\varphi_\mathcal{S},\varphi_\mathcal{T}) + \lambda)
\end{align}
\end{lemma}

\begin{proof}
\begin{align}
\mid \epsilon_\theta(h) - \epsilon_\mathcal{T}(h) \mid  &= (1 - \theta)  \mid \epsilon_\mathcal{S}(h) - \epsilon_\mathcal{T}(h) \mid 
\\ & \leq (1 - \theta) [\mid \epsilon_\mathcal{S}(h) - \epsilon_\mathcal{S}(h,h^*) \mid + \mid \epsilon_\mathcal{S}(h,h^*) - \epsilon_\mathcal{T}(h,h^*) \mid 
\nonumber \\  & \quad + \mid \epsilon_\mathcal{T}(h,h^*) - \epsilon_\mathcal{T}(h) \mid ]
\\ & \leq (1 - \theta) [\mid \epsilon_\mathcal{S}(h) - \epsilon_\mathcal{S}(h) - \epsilon_\mathcal{S}(h^*) \mid + \mid \epsilon_\mathcal{S}(h,h^*) - \epsilon_\mathcal{T}(h,h^*) \mid  
\nonumber \\ & \quad + \mid \epsilon_\mathcal{T}(h) + \epsilon_\mathcal{T}(h^*)  - \epsilon_\mathcal{T}(h) \mid ]
\\ & \leq (1 - \theta) \left[\epsilon_\mathcal{S}(h^*) + \mid \epsilon_\mathcal{S}(h,h^*) - \epsilon_\mathcal{T}(h,h^*) \mid  +  \epsilon_\mathcal{T}(h^*) \right]
\\ & \leq (1 - \theta) ( \mathcal{HW}_1(\varphi_\mathcal{S},\varphi_\mathcal{T}) + \lambda)
\end{align}
Second and third lines follow from the triangle inequality for classification error. The last line relies on Lemma \ref{Mourad}.  
\end{proof}
In this Lemma where we bound the difference between the target risk $\epsilon_\mathcal{T}(h)$ and the weighted risk $\epsilon_\theta(h)$, we show that as $\theta$ approaches 1, we rely increasingly on the target data, and the distance between domains matters less and less. 
\begin{lemma} \label{Concentration2} 
For a fixed hypothesis h, if a random labeled sample of size $n$ is generated by drawing $\vartheta m$ points from $\mu_\mathcal{T}$ and $(1 - \vartheta)m$ from $\mu_\mathcal{S}$, and labeling them according to $f_\mathcal{S}$ and $f_\mathcal{T}$, then for any $\delta \in (0,1)$, with probability at least $1 - \delta$ over the choice of the samples: 
\begin{align}
\mathbb{P}\left[\mid \hat \epsilon_\theta(h) - \epsilon_\theta(h) \mid > 2\sqrt{K/n}(\frac{\theta}{n\vartheta\sqrt{\vartheta}} + \frac{(1-\theta)}{n(1-\vartheta)\sqrt{1-\vartheta}}) + \varepsilon\right] \nonumber\\ \leq \exp\left(\frac{-\varepsilon^2n}{2K(\frac{\theta^2}{\vartheta} + \frac{(1-\theta)^2}{1-\vartheta})}\right)
\end{align}
\end{lemma}
The Lemma above from \citep{redko2017theoretical} bound the difference between the true weighted risk $\epsilon_\theta(h)$ and its empirical counterpart  $\hat{\epsilon}_\theta(h)$. 
\begin{theorem} 
Under the assumptions of Theorem \ref{TheoremMourad} and Lemma \ref{Mourad}, let $\mu_\mathcal{S},\mu_\mathcal{T} \in \mathcal{P}_p(\mathcal{X})$ and let $\varphi_\mathcal{S},\varphi_\mathcal{T} \in \mathcal{P}_p(\mathcal{P}_p(\mathcal{X}))$ such that $ \mu_\mathcal{S} = \int_{\mathcal{P}_p(\mathcal{X})} Xd\varphi_\mathcal{S}$ and $ \mu_\mathcal{T} = \int_{\mathcal{P}_p(\mathcal{X})} Xd\varphi_\mathcal{T}$, let $D$  be a labeled
sample of size $n$ with $\vartheta n$ points drawn from $\mu_\mathcal{T}$ and $(1-\vartheta)n$ from $\mu_\mathcal{S}$ with $\vartheta \in (0,1)$,
and labeled according to $f_\mathcal{S}$ and $f_\mathcal{T}$ . If $\hat h$ is the empirical minimizer of $\hat \epsilon_\theta(h)$ and $h_\mathcal{T}^* = \underset{h}{min} \, \epsilon_\mathcal{T}(h)$. Then for any $\delta \in (0,1)$, with probability at least $1 - \delta$ over the choice of the samples: 
\begin{align}
    \epsilon_\mathcal{T}(\hat h) & \leq \epsilon_\mathcal{T}(h_\mathcal{T}^*) + 2\sqrt{ \frac{2K(\frac{(1-\theta)^2}{1-\vartheta}+\frac{\theta^2}{\vartheta})\log(2/\delta)}{n}} + 4\sqrt{K/n}\left(\frac{\theta}{n\vartheta\sqrt{\vartheta}}+\frac{(1-\theta)}{n(1-\vartheta)\sqrt{1-\vartheta}}\right)
    \nonumber \\ &  \quad + 2(1-\theta)\left(\mathcal{HW}_1(\hat{\varphi}_\mathcal{S},\hat{\varphi}_\mathcal{T}) + \lambda + 2\sqrt{\frac{2\log\left(\frac{1}{\delta}\right)}{\zeta'k}}\right)
\end{align}
\end{theorem}

\begin{proof}
\begin{align}
\epsilon_\mathcal{T}(\hat h) & \leq \epsilon_\theta(\hat h) +(1-\theta)(\mathcal{HW}_1(\varphi_\mathcal{S},\varphi_\mathcal{T}) + \lambda)
\\ & \leq \hat \epsilon_\theta(\hat h) +\sqrt{ \frac{2K(\frac{(1-\theta)^2}{1-\vartheta}+\frac{\theta^2}{\vartheta})\log(2/\delta)}{n}} + 2\sqrt{K/n}\left(\frac{\theta}{n\vartheta\sqrt{\vartheta}}+\frac{(1-\theta)}{n(1-\vartheta)\sqrt{1-\vartheta}}\right) 
\nonumber \\ & \quad +(1-\theta)(\mathcal{HW}_1(\varphi_\mathcal{S},\varphi_\mathcal{T}) + \lambda)
\\ & \leq \hat \epsilon_\theta(h_\mathcal{T}^*) +\sqrt{ \frac{2K(\frac{(1-\theta)^2}{1-\vartheta}+\frac{\theta^2}{\vartheta})\log(2/\delta)}{n}} + 2\sqrt{K/n}\left(\frac{\theta}{n\vartheta\sqrt{\vartheta}}+\frac{(1-\theta)}{n(1-\vartheta)\sqrt{1-\vartheta}}\right) 
\nonumber \\ & \quad +(1-\theta)(\mathcal{HW}_1(\varphi_\mathcal{S},\varphi_\mathcal{T}) + \lambda)
+ \lambda)
\\ & \leq \epsilon_\theta(h_\mathcal{T}^*) +2\sqrt{ \frac{2K(\frac{(1-\theta)^2}{1-\vartheta}+\frac{\theta^2}{\vartheta})\log(2/\delta)}{n}} + 4\sqrt{K/n}\left(\frac{\theta}{n\vartheta\sqrt{\vartheta}}+\frac{(1-\theta)}{n(1-\vartheta)\sqrt{1-\vartheta}}\right) 
\nonumber \\ & \quad +(1-\theta)(\mathcal{HW}_1(\varphi_\mathcal{S},\varphi_\mathcal{T}) + \lambda)
\\ & \leq \epsilon_\mathcal{T}(h_\mathcal{T}^*) +2\sqrt{ \frac{2K(\frac{(1-\theta)^2}{1-\vartheta}+\frac{\theta^2}{\vartheta})\log(2/\delta)}{n}} + 4\sqrt{K/n}\left(\frac{\theta}{n\vartheta\sqrt{\vartheta}}+\frac{(1-\theta)}{n(1-\vartheta)\sqrt{1-\vartheta}}\right) 
\nonumber \\ & \quad +2(1-\theta)(\mathcal{HW}_1(\varphi_\mathcal{S},\varphi_\mathcal{T}) + \lambda)
\\ & \leq \epsilon_\mathcal{T}(h_\mathcal{T}^*) +2\sqrt{ \frac{2K(\frac{(1-\theta)^2}{1-\vartheta}+\frac{\theta^2}{\vartheta})\log(2/\delta)}{n}} + 4\sqrt{K/n}\left(\frac{\theta}{n\vartheta\sqrt{\vartheta}}+\frac{(1-\theta)}{n(1-\vartheta)\sqrt{1-\vartheta}}\right) 
\nonumber \\ & \quad +2(1-\theta)\left(\mathcal{HW}_1(\hat{\varphi}_\mathcal{S},\hat{\varphi}_\mathcal{T}) + 2\sqrt{\frac{2\log\left(\frac{1}{\delta}\right)}{\zeta'k}} + \lambda\right)
\end{align}
First and fifth lines follow from Lemma \ref{Mourad1}. Second and fourth lines are obtained using the concentration inequality of Lemma \ref{Concentration2}. Third line follows from the definition of $\hat h$ and $h_\mathcal{T}^*$. Sixth line follows from Theorem \ref{Mourad}. 
\end{proof}
This theorem demonstrates that the best hypothesis $\hat h$ that takes into account both source and target labeled data (i.e., $0 \leq \theta \leq 1$) performs at least as good as the best hypothesis $h_\mathcal{T}^*$ learned on only target data ($\theta = 1$). This result is consistent with the insight that semi-supervised domain adaptation methods are expected to be as good as or better than unsupervised methods.
\subsection{Bounds for multi-source domain adaptation}
In this section, we consider the scenario of multi-source domain adaptation, where not one but many source domains are available. More formally, we have $N$ different source domains. For each source $j$, we have a labeled sample $S_j$ of size $n_j = \vartheta_j n$ drawn from the associated unknown distribution $\mu_{S_j}$ and labeled by $f_{S_j}$,  such that $\sum_{j=1}^N \vartheta_j =1$ and $\sum_{j=1}^N n_j = n$.
\\
\\ We define the empirical weighted multi-source risk of a hypothesis $h$ for some vector $\theta = (\theta_1,...,\theta_N)$ as follows:
\begin{equation}
     \hat{\epsilon}_\theta(h) = \sum_{j=1}^N \theta_j \hat{\epsilon}_{S_j}(h) 
\end{equation}
where $\sum_{j=1}^N \theta_j =1$ and each $\theta_j$ represents the weight of the source domain $S_j$.
\\
\\We present in turn two generalization bounds for the setting of multi-source domain adaptation. The first bound uses the pairwise Hierarchical Wasserstein distance between each source and the target domain, while the second bound uses the combined Hierarchical Wasserstein distance.
\\
\\The proof of these bounds has a main common component, which we state as Lemma below. 
\begin{lemma} \label{Concentration3}
For a fixed hypothesis h, if a random labeled sample of size $n$ is generated by drawing $\vartheta_j n$ points from $\mu_{S_j}$ and labeled according to $f_{S_j}$ for each $j \in \{1,..., N\}$ and for any fixed weight vector $\theta$. Then for any $\delta \in (0,1)$, with probability at least $1 - \delta$,
\begin{equation}
\mathbb{P}\left[\mid \hat \epsilon_\theta(h) - \epsilon_\theta(h) \mid > 2\sqrt{K/n}\sum_{j=1}^N\frac{\theta_j}{\vartheta_j n \sqrt{\vartheta_j}} + \varepsilon\right] \leq \exp\left(\frac{-\varepsilon^2n}{2K \sum_{j=1}^N \frac{\theta_j^2}{\vartheta_j}}\right).
\end{equation}
\end{lemma}
This Lemma from \citep{redko2017theoretical} provides a uniform convergence bound for the empirical weighted risk. 
\subsubsection{A bound using pairwise Hierarchical Wasserstein distance}
The first bound we present considers the pairwise Hierarchical Wasserstein distance between each source and the target domain. The term $\sum_{j=1}^N \theta_j \lambda_j$ that appears in this bound plays a role corresponding to
$\lambda$ in the previous sections. 
\\
\\Before presenting the bound in question, we must prove the Lemma below that bounds the difference between the target risk $\epsilon_\mathcal{T}(h)$ and the weighted risk $\epsilon_\theta(h)$.
\begin{lemma} \label{Mourad2}
Under the assumptions of Theorem \ref{TheoremMourad} and Lemma \ref{Mourad}, let $D$ be a sample of size $n$, where for each $j \in \{1, ..., N\}$, $\vartheta_jn$ points are drawn from $\mu_{{S_j}}$ and labeled according to $f_{{S_j}}$. Then: 
\begin{equation}
\mid \epsilon_\theta(h) - \epsilon_\mathcal{T}(h) \mid \leq \sum_{j=1}^N \theta_j (\mathcal{HW}_1(\varphi_{{S_j}},\varphi_\mathcal{T}) + \lambda_j).
\end{equation}
\end{lemma}
\begin{proof} 
\begin{align}
\mid \epsilon_\theta(h) - \epsilon_\mathcal{T}(h) \mid  &= \,  \mid \sum_{j=1}^N \theta_j \epsilon_{S_j}(h) - \epsilon_\mathcal{T}(h) \mid 
\\ & \leq \sum_{j=1}^N \theta_j \mid \epsilon_{S_j}(h) - \epsilon_\mathcal{T}(h) \mid
\\ & \leq \sum_{j=1}^N \theta_j [\mid \epsilon_{S_j}(h) - \epsilon_{S_j}(h,h^*_j) \mid + \mid \epsilon_{S_j}(h,h^*_j) - \epsilon_\mathcal{T}(h,h^*_j) \mid 
\nonumber \\ & \quad + \mid \epsilon_\mathcal{T}(h,h^*_j) - \epsilon_\mathcal{T}(h)\mid]
\\ & \leq \sum_{j=1}^N \theta_j [\mid \epsilon_{S_j}(h) - \epsilon_{S_j}(h) - \epsilon_{S_j}(h^*_j)  \mid + \mid \epsilon_{S_j}(h,h^*_j) - \epsilon_\mathcal{T}(h,h^*_j) \mid
\nonumber \\ & \quad + \mid \epsilon_\mathcal{T}(h) + \epsilon_\mathcal{T}(h^*_j) - \epsilon_\mathcal{T}(h)\mid]
\\ & \leq \sum_{j=1}^N \theta_j \left[\epsilon_{S_j}(h^*_j) + \mid \epsilon_{S_j}(h,h^*_j) - \epsilon_\mathcal{T}(h,h^*_j) \mid  +  \epsilon_\mathcal{T}(h^*_j) \right]
\\ & \leq \sum_{j=1}^N \theta_j (\mathcal{HW}_1(\varphi_{{S_j}},\varphi_\mathcal{T}) + \lambda_j)
\end{align}
Third and fourth lines follow from the triangle inequality for classification error. The last line relies on Lemma \ref{Mourad}.  
\end{proof} 
We now prove the bound that considers the data available from each source individually, ignoring the relationships between sources, using pairwise Hierarchical Wasserstein distance.  
\begin{theorem} \label{MouradMultii}
Under the assumptions of Theorem \ref{TheoremMourad} and Lemma \ref{Mourad}, let $D$ be a sample of size $n$, where for each $j \in \{1, ..., N\}$, $\vartheta_jn$ points are drawn from $\mu_{S_j}$ and labeled according to $f_{S_j}$. If $\hat h$ is the empirical minimizer of $\hat \epsilon_\theta(h)$ and $h_\mathcal{T}^*=\underset{h}{\min}\epsilon_\mathcal{T}(h)$ then for any fixed $\theta$ and $\delta \in (0,1)$ with probability at least $1-\delta$ (over the choice of samples),
\begin{align}
    \epsilon_\mathcal{T}(\hat h) & \leq \epsilon_\mathcal{T}(h_\mathcal{T}^*) + 2\sqrt{ \frac{2K\sum_{j=1}^N\frac{\theta_j^2}{\vartheta_j}\log(2/\delta)}{n}} + 
    2\sqrt{\sum_{j=1}^N\frac{K\theta_j}{\vartheta_jn}} 
    \nonumber \\ & \quad + 2\sum_{j=1}^N \theta_j\left(\mathcal{HW}_1(\hat{\varphi}_{S_j},\hat{\varphi}_\mathcal{T}) + \lambda_j + 2\sqrt{\frac{2\log\left(\frac{1}{\delta}\right)}{\zeta'k}}\right),
\end{align}
where $\lambda_j=\underset{h}{min} (\epsilon_{S_j}(h)+\epsilon_\mathcal{T}(h))$ represents the joint error for each source domain ${S_j}$.
\end{theorem}
\begin{proof}
\begin{align}
\epsilon_\mathcal{T}(\hat h) & \leq \epsilon_\theta(\hat h) +\sum_{j=1}^N \theta_j (\mathcal{HW}_1(\varphi_{S_j},\varphi_\mathcal{T}) + \lambda_j)
\\ & \leq \hat \epsilon_\theta(\hat h) +\sqrt{ \frac{2K\sum_{j=1}^N\frac{\theta_j^2}{\vartheta_j}\log(2/\delta)}{n}} + \sqrt{\sum_{j=1}^N\frac{K\theta_j}{\vartheta_jn}}
\nonumber \\ & \quad +\sum_{j=1}^N \theta_j (\mathcal{HW}_1(\varphi_{S_j},\varphi_\mathcal{T}) + \lambda_j)
\\ & \leq \hat \epsilon_\theta(h_\mathcal{T}^*) +\sqrt{ \frac{2K\sum_{j=1}^N\frac{\theta_j^2}{\vartheta_j}\log(2/\delta)}{n}} + \sqrt{\sum_{j=1}^N\frac{K\theta_j}{\vartheta_jn}}
\nonumber \\ & \quad +\sum_{j=1}^N \theta_j (\mathcal{HW}_1(\varphi_{S_j},\varphi_\mathcal{T}) + \lambda_j)
\\ & \leq \epsilon_\theta(h_\mathcal{T}^*) +2\sqrt{ \frac{2K\sum_{j=1}^N\frac{\theta_j^2}{\vartheta_j}\log(2/\delta)}{n}} + 2\sqrt{\sum_{j=1}^N\frac{K\theta_j}{\vartheta_jn}}
\nonumber \\ & \quad +\sum_{j=1}^N \theta_j (\mathcal{HW}_1(\varphi_{S_j},\varphi_\mathcal{T}) + \lambda_j)
\\ & \leq \epsilon_\mathcal{T}(h_\mathcal{T}^*) +2\sqrt{ \frac{2K\sum_{j=1}^N\frac{\theta_j^2}{\vartheta_j}\log(2/\delta)}{n}} + 2\sqrt{\sum_{j=1}^N\frac{K\theta_j}{\vartheta_jn}}
\nonumber \\ & \quad +2\sum_{j=1}^N \theta_j (\mathcal{HW}_1(\varphi_{S_j},\varphi_\mathcal{T}) + \lambda_j)
\\ & \leq \epsilon_\mathcal{T}(h_\mathcal{T}^*) +2\sqrt{ \frac{2K\sum_{j=1}^N\frac{\theta_j^2}{\vartheta_j}\log(2/\delta)}{n}} + 2\sqrt{\sum_{j=1}^N\frac{K\theta_j}{\vartheta_jn}}
\nonumber \\ & \quad  +2\sum_{j=1}^N \theta_j \left(\mathcal{HW}_1(\hat{\varphi}_\mathcal{S},\hat{\varphi}_\mathcal{T}) + 2\sqrt{\frac{2\log\left(\frac{1}{\delta}\right)}{\zeta'k}} + \lambda_j \right)
\end{align}
First and fifth lines follow from Lemma \ref{Mourad2}. Second and fourth lines are obtained using the concentration inequality of Lemma \ref{Concentration3}. Fourth line is a consequence of Lemma \ref{Mourad}. Third line follows from the definition of $\hat h$ and $h_\mathcal{T}^*$. Sixth line follows from Theorem \ref{Mourad}. 
\end{proof}
\subsubsection{A bound using combined Hierarchical Wasserstein distance}
In the former bound, the Hierarchical Wasserstein distance between domains is only measured on pair, so it is not required to have a hypothesis that is valid for each source domain. The alternate bound shown in the next theorem enables us to alter the source distribution by changing $\theta$. This has two implications. First of all, we now need to insist that there is a hypothesis $h^*$ which has low risk on both the $\theta$-weighted convex combination of sources and the target domain. Secondly, we measure the Hierarchical Wasserstein distance between the target and a mixture of sources, instead of between the target and every single source.
\begin{lemma} \label{Mourad3} Under the assumptions of Theorem \ref{TheoremMourad} and Lemma \ref{Mourad}, let $D$ be a sample of size $n$, where for each $j \in \{1, ..., N\}$, $\vartheta_jn$ points are drawn from $\mu_{S_j}$ and labeled according to $f_{S_j}$. Then
\begin{align}
\mid \epsilon_\theta(h) - \epsilon_\mathcal{T}(h) \mid \leq \mathcal{HW}_1(\varphi_{S_\theta},\varphi_\mathcal{T}) + \lambda_\theta.
\end{align}
\end{lemma}

\begin{proof}
\begin{align}
\mid \epsilon_\theta(h) - \epsilon_\mathcal{T}(h) \mid  & \leq \mid  \epsilon_\theta(h) - \epsilon_\theta(h,h^*) \mid + \mid \epsilon_\theta(h,h^*) - \epsilon_\mathcal{T}(h,h^*) \mid 
\nonumber \\ & \quad + \mid \epsilon_\mathcal{T}(h,h^*) - \epsilon_\mathcal{T}(h)\mid
\\ & \leq \mid  \epsilon_\theta(h) - \epsilon_\theta(h) + - \epsilon_\theta(h^*) \mid + \mid \epsilon_\theta(h,h^*) - \epsilon_\mathcal{T}(h,h^*) \mid 
\nonumber \\ & \quad + \mid \epsilon_\mathcal{T}(h) + \epsilon_\mathcal{T}(h^*) - \epsilon_\mathcal{T}(h)\mid
\\ & \leq \epsilon_\theta(h^*) + \mid \epsilon_\theta(h,h^*) - \epsilon_\mathcal{T}(h,h^*) \mid + \epsilon_\mathcal{T}(h^*) 
\\ & \leq \mathcal{HW}_1(\varphi_{S_\theta},\varphi_\mathcal{T}) + \lambda_\theta
\end{align}
First and second lines follow from the triangle inequality for classification error. The last line relies on Lemma \ref{Mourad}. 
\end{proof}
We now prove the bound using combined Hierarchical Wasserstein distance. 
\begin{theorem} \label{MouradMulti2}
Under the assumptions of Theorem \ref{TheoremMourad} and Lemma \ref{Mourad}, let $D$ be a sample of size $n$, where for each $j \in \{1, ..., N\}$, $\vartheta_jn$ points are drawn from $\mu_{S_j}$ and labeled according to $f_{S_j}$. If $\hat h$ is the empirical minimizer of $\hat \epsilon_\theta(h)$ and $h_\mathcal{T}^*=\underset{h}{\min}\epsilon_\mathcal{T}(h)$ then for any fixed $\theta$ and $\delta \in (0,1)$ with probability at least $1-\delta$ (over the choice of samples)
\begin{align}
    \epsilon_\mathcal{T}(\hat h) & \leq \epsilon_\mathcal{T}(h_\mathcal{T}^*) + 2\sqrt{ \frac{2K\sum_{j=1}^N\frac{\theta_j^2}{\vartheta_j}\log(2/\delta)}{n}} + 2\sqrt{\sum_{j=1}^N\frac{K\theta_j}{\vartheta_jn}} 
    \nonumber \\ & \quad +2\left(\mathcal{HW}_1(\varphi_{S_\theta},\varphi_\mathcal{T}) + \lambda_\theta +  2\sqrt{\frac{2\log\left(\frac{1}{\delta}\right)}{\zeta'k}}\right).
\end{align}
where $\lambda_j=\underset{h}{min} (\epsilon_{S_\theta}(h)+\epsilon_\mathcal{T}(h))$ represents the joint error for each source domain ${S_j}$.
\end{theorem}
\begin{proof}
\begin{align}
\epsilon_\mathcal{T}(\hat h) & \leq \epsilon_\theta(\hat h) + \mathcal{HW}_1(\varphi_{S_\theta},\varphi_\mathcal{T}) + \lambda_\theta
\\ & \leq \hat \epsilon_\theta(\hat h) +\sqrt{ \frac{2K\sum_{j=1}^N\frac{\theta_j^2}{\vartheta_j}\log(2/\delta)}{n}} + \sqrt{\sum_{j=1}^N\frac{K\theta_j}{\vartheta_jn}}
\nonumber \\ & \quad  + \mathcal{HW}_1(\varphi_{S_\theta},\varphi_\mathcal{T}) + \lambda_\theta
\\ & \leq \hat \epsilon_\theta(h_\mathcal{T}^*) +\sqrt{ \frac{2K\sum_{j=1}^N\frac{\theta_j^2}{\vartheta_j}\log(2/\delta)}{n}} + \sqrt{\sum_{j=1}^N\frac{K\theta_j}{\vartheta_jn}}
\nonumber \\ & \quad  +\mathcal{HW}_1(\varphi_{S_\theta},\varphi_\mathcal{T}) + \lambda_\theta
\\ & \leq \epsilon_\theta(h_\mathcal{T}^*) +2\sqrt{ \frac{2K\sum_{j=1}^N\frac{\theta_j^2}{\vartheta_j}\log(2/\delta)}{n}} + 2\sqrt{\sum_{j=1}^N\frac{K\theta_j}{\vartheta_jn}}
\nonumber \\ & \quad  +\mathcal{HW}_1(\varphi_{S_\theta},\varphi_\mathcal{T}) + \lambda_\theta
\\ & \leq \epsilon_\mathcal{T}(h_\mathcal{T}^*) +2\sqrt{ \frac{2K\sum_{j=1}^N\frac{\theta_j^2}{\vartheta_j}\log(2/\delta)}{n}} + 2\sqrt{\sum_{j=1}^N\frac{K\theta_j}{\vartheta_jn}}
\nonumber \\ & \quad +2\left(\mathcal{HW}_1(\varphi_{S_\theta},\varphi_\mathcal{T}) + \lambda_\theta\right)
\\ & \leq \epsilon_\mathcal{T}(h_\mathcal{T}^*) +2\sqrt{ \frac{2K\sum_{j=1}^N\frac{\theta_j^2}{\vartheta_j}\log(2/\delta)}{n}} + 2\sqrt{\sum_{j=1}^N\frac{K\theta_j}{\vartheta_jn}}
\nonumber \\ & \quad  +2\left(\mathcal{HW}_1(\varphi_{S_\theta},\varphi_\mathcal{T})  + 2\sqrt{\frac{2\log\left(\frac{1}{\delta}\right)}{\zeta'k}} + \lambda_\theta\right)
\end{align}
First and fifth lines follow from Lemma \ref{Mourad3}. Second and fourth lines are obtained using the concentration inequality of Theorem \ref{Concentration3}. Third line follows from the definition of $\hat h$ and $h_\mathcal{T}^*$. Sixth line follows from Theorem \ref{Mourad}. 
\end{proof}
\newpage
\section{Conclusion}
In this paper, using hierarchical optimal transport we presented a theoretical study of domain adaptation, a problem in which we have an abundant amount of labeled samples from a source domain, but we aim to deploy a model in another target domain with a much smaller amount of labeled samples or even no labeled samples. Our main results are generalization bounds for both single and multi-source domain adaptation scenarios, where the divergence between the source and target domains is measured by the Hierarchical Wasserstein distance. Our generalization bounds justify the application of hierarchical optimal transport in the context of domain adaptation and may suggest under the assumption of successful clustering in the target domain that one can minimize the other terms $\epsilon_\mathcal{S}(h)$ and $\lambda$ without difficulty independently from the minimization of the Hierarchical Wasserstein distance. 
\\
\\ Future perspectives of this work are numerous and concern both the derivation of new domain adaptation algorithms and the demonstration of new generalization bounds. Indeed, the work of this chapter can be extended in different directions:
\begin{itemize}
    \item First of all, we would like to derive a new domain adaptation algorithm based on the insights provided by the bounds in the multi-source settings.
    \item Secondly, we aim to produce new generalization bounds that take into account the quality of clustering in the target domain, by reflecting explicitly the excess clustering risk.
    \item Finally, the ability term $\lambda$ is surprisingly understudied, and we would like to provide a theoretical analysis of this term using hierarchical optimal transport and investigate the possibility to estimate it from finite samples.
\end{itemize}
\bibliography{sample}

\end{document}